\documentclass[letterpaper, 10 pt, conference]{ieeeconf}  % Comment this line out
\IEEEoverridecommandlockouts
\let\labelindent\relax
\usepackage{amsmath,amsfonts,amssymb,amsthm}
\usepackage{bm}
\usepackage[utf8]{inputenc}
\usepackage{verbatim}
\usepackage{graphicx}
\usepackage{cite}
\usepackage{hyperref}
\usepackage{enumitem}
\usepackage{xcolor}
\usepackage{vec_notation}
\usepackage{hanch_sty}
\newtheorem{theorem}{Theorem}

\newtheorem{corollary}[theorem]{Corollary}
\newtheorem{proposition}[theorem]{Proposition}

\newcommand{\Mid}{G}
\newcommand{\Side}{S}

\newcommand{\diag}{\mathrm{diag}}

\def\BibTeX{{\rm B\kern-.05em{\sc i\kern-.025em b}\kern-.08em
    T\kern-.1667em\lower.7ex\hbox{E}\kern-.125emX}}

\title{\LARGE Understanding Incremental Learning with Closed-form Solution to\\ Gradient Flow on Overparamerterized Matrix Factorization}

\author{Hancheng Min and Ren\'e Vidal
\thanks{This work was done when H. Min was a postdoc at the University of Pennsylvania. The authors acknowledge the support of the NSF under grants 2031985 and 2212457, the Simons Foundation under grant 814201, and the ONR MURI Program under grant 503405-78051. H. Min thanks  Arthur C. B. de Oliveira for the insightful discussion during the early stage of this work.} 
\thanks{H. Min is with the Institute of Natural Sciences (INS) and the School of Mathematical Sciences (SMS) at the Shanghai Jiao Tong University. \texttt{hanchmin@sjtu.edu.cn}. }
\thanks{R. Vidal is with the Department of Electrical and Systems Engineering (ESE) and the Department of Radiology in the  Perelman School of Medicine at the University of Pennsylvania. \texttt{vidalr@upenn.edu}
}}
\begin{document}
\maketitle
\thispagestyle{empty}
\pagestyle{empty}

\begin{abstract}
Many theoretical studies on neural networks attribute their excellent empirical performance to the implicit bias or regularization induced by first-order optimization algorithms when training networks under certain initialization assumptions. One example is the incremental learning phenomenon in gradient flow (GF) on an overparamerterized matrix factorization problem with small initialization: GF learns a target matrix by sequentially learning its singular values in decreasing order of magnitude over time. In this paper, we develop a quantitative understanding of this incremental learning behavior for GF on the symmetric matrix factorization problem, using its closed-form solution obtained by solving a Riccati-like matrix differential equation. We show that incremental learning emerges from some time-scale separation among dynamics corresponding to learning different components in the target matrix. By decreasing the initialization scale, these time-scale separations become more prominent, allowing one to find low-rank approximations of the target matrix. Lastly, we discuss the possible avenues for extending this analysis to asymmetric matrix factorization problems.
\end{abstract}

\section{Introduction}
Why do deep neural networks~\cite{krizhevsky2012imagenet,he2016deep,vaswani2017attention} work so well in practice? At the center of this mystery lies the fact that  practical neural networks are typically \emph{overparameterized}, i.e., 
their number of parameters is several orders of magnitude larger than the number of available training examples. As a consequence, there are infinitely many parameter choices that can perfectly fit the data. However, different parameters have different generalization properties, and only a few choices lead to networks that make correct predictions for new samples.
%having several magnitudes more parameters than the number of available samples for training it, thereby admitting infinite choices for the parameters under which the network can fit these samples perfectly but with varying abilities to generalize, that is, to make correct prediction over new test samples. 
Behind the success of deep learning lies the ability of training algorithms, such as gradient descent, to find those networks that generalize well, which over past years has motivated many theoretical studies on such an \emph{implicit bias}~\cite{vardi2023implicit} of learning algorithms. 

An interesting one is the \emph{simplicity bias} of gradient descent or gradient flow under certain initial conditions, typically random initialization with small variance, through which the learned network model possesses certain notions of low complexity. Depending on the network architecture, such notions range from classic ones, such as low-rankness~\cite{gunasekar2017implicit,Gidel2019,pmlr-v202-jin23a,min2023early} or sparsity~\cite{gunasekar2018convo,moroshko2020implicit,li2021sparse} of the network weights, to complexity measures in some function space \cite{abbe2023transformers}. In many scenarios, the analyses for such simplicity biases often call for theoretical understandings of the \emph{incremental learning}~\cite{saxe2014exact,Gidel2019,pmlr-v202-jin23a} phenomenon, which describes the evolution of the network model during training as sequential learning of components of the underlying target ground truth function (suppose there is one) in some canonical decompositions, without learning any spurious components that overfit to the training data.

One important testbed for understanding incremental learning is the \emph{matrix factorization}~\cite{Burer2005} problem, studied in the context of linear regression~\cite{saxe2014exact,Gidel2019,pmlr-v139-tarmoun21a,mtvm21icml,de2023dynamics}, matrix sensing~\cite{gunasekar2017implicit,arora2019implicit,stoger2021small,pmlr-v202-jin23a}, and recently the low-rank adaptation~\cite{yaras2024compressible,xu2025understanding}. Under random initialization with small variance, training with gradient descent/flow on the factorized model finds a stationary point that corresponds exactly (noiseless) or approximately (noisy) to the underlying ground truth matrix, following a particular dynamic trajectory. During training, the factorized model grows its singular components to fit/learn the ground truth matrix in a sequential manner such that the growth happens for one singular component at a time to fit one component in the ground truth, and the ground truth component with the largest singular values get learned first, then followed by the second largest and so on so forth. As such, the training algorithm finds the ground truth by every time \emph{incrementally} learning one of its components.

A body of theoretical work on incremental learning in matrix factorization is motivated by some of the early works~\cite{saxe2014exact,gunasekar2017implicit} that investigate the implicit bias of deep learning algorithms. One initial line of attempts studies  \emph{closed-form} solutions: If one carefully initializes the singular vectors of the factorized model to align with those of the ground truth, then the training dynamics are decoupled into multiple scalar dynamics, each corresponding to one principal component in the ground truth~\cite{saxe2014exact,Gidel2019,pmlr-v139-tarmoun21a}. In this case, understanding incremental learning reduces to studying the time-scale separation among these scalar dynamics. However, these results are achieved under very restrictive conditions on initialization. 

To analyze this phenomenon under general initial conditions, recent work focuses on a more fine-grained \emph{trajectory analysis}\cite{stoger2021small,pmlr-v202-jin23a} that introduces some orthogonal decomposition of the parameter space into signal and error subspaces based on the ground truth. This analysis shows that, at certain time epoch, the components within the signal subspace learn some principal components of the ground truth, together with upper bounds on the growth within the error subspace. While the trajectory analysis produces the most general results~\cite{pmlr-v202-jin23a} for matrix factorization, it does not fully answer many quantitative questions. 
For example, how small does the variance of the random initialization need to be in order for incremental learning to happen?
%For example, under how small the variance of the random initialization, does incremental learning happen? 
If one seeks one of the low-rank approximations of the ground truth, within what time interval does the algorithm need to stop? These questions have become increasingly important alongside the growing interest (for example, \cite{chou2025get}) in utilizing the incremental mechanism to obtain low-complexity models, as opposed to the classic regularized approach. 

In this paper, we extend the closed-form solution-based line of work to gradient flow on matrix factorization under general initial conditions, facilitating a rich quantitative understanding of the incremental phenomenon. Our theorems aim to exactly address several questions previously mentioned by showing the precise dependence of incremental learning on the variance of the random initialization and providing time intervals within which the factorized model serves as a good low-rank approximation of the ground truth. Although our results are primarily for symmetric matrix factorization, we outline the path toward extending them to the asymmetric matrix factorization problems, opening up new opportunities to utilize closed-form solutions to study a broader class of problems. 

The organization of this paper is as follows. In Section \ref{sec_prelim}, we introduce the gradient flow dynamics on the symmetric matrix factorization problem, the primary focus of our theoretical results, and we discuss the closed-form solution of the factorized model to the gradient flow. This forms the basis of our main results in Section \ref{sec_main} on the incremental learning in the symmetric matrix factorization problem. Then, in Section \ref{sec_conclusion}, we conclude by briefly discussing the path toward extending current results to the asymmetric matrix factorization problems.

\emph{Notations}: For an $n\by m$ matrix $A$, we let $\|A\|$ and $\|A\|_F$ denote the spectral and Frobenius norm of $A$, respectively. We write $A\succeq 0$ ($A\preceq 0$) when $A$ is positive semi-definite (negative semi-definite). Then, we let $\diag\{a_i\}_{i=1}^n$ be a diagonal matrix with $a_i$ being its $i$-th diagonal entry.
% We define $\one_A$ as the indicator for a statement $A$: $\one_A=1$ if $A$ is true and $\one_A=0$ otherwise, and define $[\cdot]_+:=\max\{\cdot,0\}$. We also let $\mathcal{N}(\bmu,\bm{\Sigma}^2)$ denote the normal distribution with mean $\bmu$ and covariance matrix $\bm{\Sigma}^2$, and $\text{Unif}(S)$ denote the uniform distribution over a set $S$. Lastly, we let $[N]$ denote the integer set $\{1,\cdots,N\}$.

\section{Preliminaries: Gradient Flow on Symmetric Matrix Factorization}\label{sec_prelim}
The problem of our main interest in this paper is the following \emph{symmetric matrix factorization}:
\be
    \min_{U\in \mathbb{R}^{n\times r}} \mathcal{L}(U)=\frac{1}{4}\|Y-UU^\top\|^2_F\,,\label{eq_sym_mf}
\ee
where $Y\in\mathbb{R}^{n\times n}$ with $Y\succeq 0$ is the \emph{target matrix} to be factorized. Although one may find solving \eqref{eq_sym_mf} easy by simply taking the SVD of $Y$, the goal of this paper is rather to understand training trajectories when one seeks to solve \eqref{eq_sym_mf} by gradient flow (GF) on $\mathcal{L}(U)$, i.e.,  
\be
    \dot{U}=-\nabla_U\mathcal{L}(U)=(Y-UU^\top)U,\ U(0)=U_0\,,\label{eq_sym_gf_u}
\ee
from some initial condition $U_0$. More specifically, as stated in the introduction, we are interested in the incremental learning phenomenon when $U(t)U^\top(t)$ learns the singular components of $Y$ sequentially along the GF trajectory. Therefore, it suffices to study the induced dynamic evolution of the factorized model $U(t)U^\top(t):=W(t)$ from \eqref{eq_sym_gf_u}, which is
\begin{align}
    &\dot{W}\!=\!\dot{U}U^\top\!\!+\!U\dot{U}^\top\!=\!WY\!+\!YW\!-\!2W^2, W\!(0)\!=\!U_0U_0^\top\!.\label{eq_sym_mf_w}
\end{align}
\eqref{eq_sym_mf_w} is a special form of matrix Riccati differential equations, thus can be analyzed from its closed-form solution.

\subsection{Closed-form solution}
The closed-form solutions to matrix Riccati differential equations are well-studied, for example, in~\cite{kuvcera1973review,Sasagawa1982finite}. We thus have the following result regarding the solution to \eqref{eq_sym_mf_w}
\begin{proposition}\label{prop_cl_sol}
    If $Y$ has $K$ non-zero singular values and the full SVD of $Y\succeq 0$ is $\Phi\begin{bmatrix}
    \Sigma_Y & 0\\
    0 &0
\end{bmatrix}\Phi^\top$, then 
    the matrix Riccati differential equation 
    \begin{equation}
        \dot{W}=WY+YW-2W^2,\ W(0)=W_0\succeq 0
    \end{equation}
    has unique solution of the following:
    \begin{equation}
        W(t)\!=\!\Phi \Side(t) \tilde{W}_0\left(I_n\!+\!\Mid (t) \tilde{W}_0\right)^{-1}\Side^\top\!(t)\Phi^\top,\label{eq_riccati_sol}
    \end{equation}
    where $\tilde{W}_0=\Phi^\top W_0 \Phi$ and 
    \be
        \Mid (t)\!=\!\!\begin{bmatrix}
            \Sigma_Y^{-1}\!(e^{2\Sigma_Yt}\!-\!I_{K})\!\!&0\\
            0& \!\!2I_{n\!-\!K}t
        \end{bmatrix}\!, \Side(t)\!=\!\!\begin{bmatrix}
            e^{\Sigma_Yt}\!\! & 0 \\
            0& \!\!I_{n\!-\!K}
        \end{bmatrix}\!.\label{eq_aux_mat}
    \ee
\end{proposition}
\begin{proof}
    From~\cite{kuvcera1973review,Sasagawa1982finite}, one knows that the solution to general matrix Riccati equation
    \begin{equation}
        \dot{P}=AP+PA^\top-PRP+Q,\ P(0)=P_0\,,
    \end{equation}
    is given by $P(t)=X_2(t)X_1^{-1}(t)$, where $\{X_1(t),X_2(t)\}$ are solution to an LTI system
    \begin{equation}
        \begin{bmatrix}
            \dot{X}_1\\
            \dot{X}_2
        \end{bmatrix}\!=\!\begin{bmatrix}
            -A^\top&R\\
            Q& A
        \end{bmatrix}\!\begin{bmatrix}
            X_1\\
            X_2
        \end{bmatrix},\ X_1(0)=I_n,X_2(0)=P_0\,,\label{eq_aug_lti}
    \end{equation}
    To apply this, let $P\!=\!\Phi^\top\! W\! \Phi$, $A\!=\!\Sigma_Y$, $Q\!=\!0$ and $R\!=\!I_n$, by solving \eqref{eq_aug_lti} we have (an equivalent expression for \eqref{eq_riccati_sol})
    \be
        \underbrace{\Phi^\top\! W(t) \Phi}_{P(t)}\!=\!\underbrace{\Side(t) \Phi^\top\! W_0\Phi}_{X_2(t)}\! \cdot\! \bigg(\underbrace{\!\Side^{-1\!}(t)\!\!\lp I_n\!+\!\Mid(t) \Phi^\top W_0\Phi\rp}_{X_1(t)}\!\!\bigg)^{\!\!\!\!-1}\!\!\!\!,\!\!\label{eq_riccati_sol_temp}
    \ee
    where $\Mid(t),\Side(t)$ are defined in \eqref{eq_aux_mat}. Notice that $\forall t$, all eigenvalues of $G(t)W_0$ are real and non-negative since both $G(t)$ and $W_0$ are positive semi-definite, then $X_1(t)$ is invertible $\forall t$, ensuring \eqref{eq_riccati_sol_temp} is the unique solution~\cite[Corollary 1]{Sasagawa1982finite}.
\end{proof}
\subsection{Additional settings}
With the closed-form solution available, one can study this GF under any initial condition $W_0=U_0U_0^\top\succeq 0$. However, different types of initialization (for example, random v.s. deterministic initialization) may require different analyses. Due to space constraints, we opt to discuss the case of \emph{overparametrized} matrix factorization under deterministic initialization with varying scales.
\begin{figure*}[ht!]
    \centering
    \includegraphics[width=0.98\linewidth]{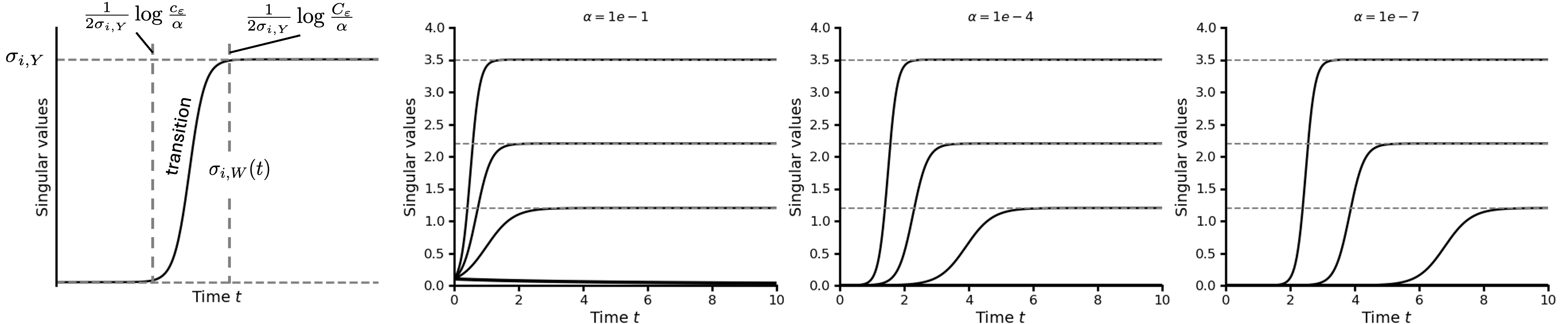}
    \caption{Incremental learning in symmetric matrix factorization. Under small spectral initialization, (left plot) the dynamic evolution of $\sigma_{i,W}(t)$ for learning singular component $\sigma_{i,Y}$ in target matrix $Y$ has a sharp transition phase where its value grows rapidly from approximately zero to one that is close to the target $\sigma_{i,Y}$. (Right three plots) As the initialization scale $\alpha$ decreases, the transition phases for different target values get further separated, and low-rank approximations of $Y$ emerge along the GF trajectory. Similar dynamic evolutions of $W(t)$ still exist under general initialization, as we show in Theorem \ref{thm_gen_inc}.}
    \label{fig:inc}
\end{figure*}

As we described in the introduction, the incremental learning phenomenon generally happens under initialization with a small scale. To isolate this critical effect of the scale/norm of the initialization, it is convenient to write our initial condition as $U(0)=U_0=\alpha^{1/2} \bar{U}_0$ for some \emph{initialization scale} $\alpha>0$ and \emph{initialization shape} $\bar{U}_0\succeq 0$.  Our analysis 
\begin{itemize}[leftmargin=0.35cm]
    \item assumes a fixed shape $\bar{U}_0$ and describes how the scale $\alpha$ affects the incremental learning phenomenon;
    \item In addition, we assume $W_0= \alpha \bar{U}_0\bar{U}_0^\top$ is invertible. This necessarily requires that the problem \eqref{eq_sym_mf} is \emph{overparametrized}~\cite{pmlr-v139-tarmoun21a}, i.e. $r\geq n$, which imply that the optimal value of \eqref{eq_sym_mf} is always zero (target is exactly factorized) regardless of the target $Y$.
\end{itemize}

\section{Main results: Incremental Learning in Symmetric Matrix Factorization}\label{sec_main}
In this section, we discuss incremental learning in symmetric factorization problems. First, we revisit the GF under spectral initialization with a small scale and refine the existing asymptotic results~\cite{Gidel2019} to a non-asymptotic one. Then, we show that a similar quantitative characterization of incremental learning holds for general initialization. These main results follow the problem settings stated in the previous section, and afterward, we discuss how to adapt them to other settings.
\subsection{Incremental learning under spectral initialization}
Give the GF \eqref{eq_sym_gf_u} with the SVD of the target matrix $Y=\Phi \Sigma_Y \Phi^\top$, the spectral initialization of $U$ is any $U_0$ that has an SVD of the form $U_0=\alpha^{1/2}\Phi \Sigma_{U_0} V_{U_0}^\top$. Recall that $\alpha>0$ is the initialization scale. Under a spectral initialization, the induced ode on $W$ \eqref{eq_sym_mf_w} becomes:
\be
    \dot{W}=WY+YW-2W^2,\ W(0)=\alpha\Phi \Sigma_{U_0}^2\Phi^\top\,,
\ee
Through a change of variables $\tilde{W}=\Phi^\top W\Phi$, the dynamics on $\tilde{W}$ is given by
\be
    \dot{\tilde{W}}=\tilde{W}\Sigma_Y+\Sigma_Y\tilde{W}-2\tilde{W}^2,\ \tilde{W}(0)=\alpha\Sigma_{U_0}^2\,,
\ee
whose solution $\tilde{W}(t)$ is diagonal $\forall t$, and the dynamics of diagonal entries are decoupled from each other. Indeed, examining the closed-form solution in Proposition \ref{prop_cl_sol} under a spectral initialization, we have the following:

\begin{corollary}\label{col_spec_sol}
    Suppose $\Sigma_{Y}=\diag\{\sigma_{i,Y}\}_{i=1}^n$. Under some spectral initialization $U(0)=\alpha^{1/2}\Phi\Sigma_{U_0}V_{U_0}^\top$, the solution $U(t)$ to the GF dynamics \eqref{eq_sym_gf_u} induces $W(t)=U(t)U^\top(t)$ of the form $W(t)=\Phi \diag\{\sigma_{i,W}(t)\}_{i=1}^n\Phi^\top$ with
    \begin{align}
        \sigma_{i,W}(t)&\!=\!\frac{\alpha\sigma_{i,Y}\sigma_{i,0}e^{2\sigma_{i,Y}t}}{\sigma_{i,Y}+\alpha\sigma_{i,0}(e^{2\sigma_{i,Y}t}-1)}& \,, \text{ if } \sigma_{i,Y}\neq 0; \label{eq_sv_traj_si}\\
        \sigma_{i,W}(t)&\!=\!\frac{\alpha\sigma_{i,0}}{1+2\alpha\sigma_{i,0}t}&\,,  \text{ if }\sigma_{i,Y}=0\,,\label{eq_sv_traj_si_null}
    \end{align}
    where $\sigma_{i,0}=[\Sigma_{U_0}^2]_{ii}\geq 0,\forall i$.
\end{corollary}
Spectral initialization offers a simplified analysis of the GF trajectory. Under the induced dynamics \eqref{eq_sym_mf_w}, we expect the factorized model $W(t)$ to converge to the target matrix $Y$. Indeed, for each $i$, \eqref{eq_sv_traj_si} is the learning trajectory for a non-zero singular mode $\sigma_{i,Y}$ of $Y$, and it is easy to see that as long as $\sigma_{i,0}\neq 0$, we have $\lim_{t\ra \infty} \sigma_{i,W}(t)=\sigma_{i,Y}$; \eqref{eq_sv_traj_si_null} is the learning trajectory for a zero singular mode $\sigma_{i,Y}$ of $Y$, and we always have $\lim_{t\ra \infty} \sigma_{i,W}(t)=0$. Therefore, if $\sigma_{i,0}\neq 0, \forall i$ with $\sigma_{i,Y}\neq 0$, one has $\lim_{t\ra \infty} W(t)=Y$, thereby learning the target matrix. Note that a similar condition is also necessary in the general case. That is, learning the exact $Y$ requires $\Phi_Y^\top W(0)\Phi_Y$ to have full rank, where columns of $\Phi_Y$ form an orthonormal basis of the range space of $Y$. 

As shown in the previous discussion, the asymptotic behavior of these scalar dynamics \eqref{eq_sv_traj_si}\eqref{eq_sv_traj_si_null} are easy to study, concerning whether $W(t)$ can successfully learn the target matrix $Y$. Now, we turn to the questions regarding the \emph{transient behavior} under small initialization with scale $\alpha\ll 1$. Notice that when $0<\alpha\sigma_{i,0}<\sigma_{i,Y}$, $\sigma_{i,W}(t)$ is strictly monotonically increasing w.r.t. time $t$. Our following discussions assume $\alpha$ is sufficiently small so that $\sigma_{i,W}(t)$ is monotonically increasing. 

For every $\sigma_{i,W}(t)$ following \eqref{eq_sv_traj_si} that learns a non-zero mode $\sigma_{i,Y}$ of $Y$, we have $\sigma_{i,W}(0)=\alpha\sigma_{i,0}$, which is close to zero for sufficiently small $\alpha$, and we know asymptotically $\lim_{t\ra \infty} \sigma_{i,W}(t)=\sigma_{i,Y}$. Therefore, each $\sigma_{i,W}(t)$ must learn its target $\sigma_{i,Y}$ through some transitional phase during which the value $\sigma_{i,W}(t)$ grow from a small value close to $\alpha\sigma_{i,0}$ to a value close to $\sigma_{i,Y}$. If there is a \emph{time-scale separation} between these transitional phases for different $i$, then we expect that at some intermediate time along GF, a subset of $\sigma_{i, W}(t)$ has learned their corresponding target $\sigma_{i,Y}$ while the rest stays close to zero, from which the resulting $W(t)$ serves as some low-rank approximation of $Y$.

The above discussion suggests a closer look at the transient behavior in the solution \eqref{eq_sv_traj_si} when the initialization scale is small: For some $c,C>0$ such that $c\sigma_{i,0}\ll \sigma_{i,Y}$ and $C\sigma_{i,0}\gg \sigma_{i,Y}$, $\forall i$. Then we see that for every $i$ (for any $\alpha\leq c$),
\begin{align}
    \sigma_{i,W}\left(\frac{1}{2\sigma_{i,Y}}\log\frac{c}{\alpha}\right)&=\frac{c\sigma_{i,Y}\sigma_{i,0}}{\sigma_{i,Y}+(c-\alpha)\sigma_{i,0}}\simeq c\sigma_{i,0},\label{eq_sv_traj_si_eval_s}\\
    \sigma_{i,W}\left(\frac{1}{2\sigma_{i,Y}}\log\frac{C}{\alpha}\right)&=\frac{C\sigma_{i,Y}\sigma_{i,0}}{\sigma_{i,Y}+(C-\alpha)\sigma_{i,0}}\simeq \sigma_{i,Y}\,,\label{eq_sv_traj_si_eval_b}
\end{align}
suggesting a sharp transition in $\sigma_{i,\W}(t)$ from a small value $c\sigma_{i,0}$ to one close to its target value $\sigma_{i,Y}$ at a time that scales as $\Theta\left(\frac{1}{\sigma_{i,Y}}\log\frac{1}{\alpha}\right)$. Notably, the transition time depends inverse-proportionally on $\sigma_{i,Y}$ (See Figure \ref{fig:inc}). Thus, this time-scale separation can result in singular values of $Y$ being learned one by one, with larger ones learned first and smaller ones later, which is exactly an incremental learning phenomenon. Based on these discussions, the conditions to achieve this are:
\begin{itemize}[leftmargin=0.35cm]
    \item Singular values $\sigma_{i,Y}$ of $Y$ are distinct;
    \item The scale $\alpha$ is sufficiently small such that the transitional phases $\lhp \frac{1}{2\sigma_{i,Y}}\log\frac{c}{\alpha}, \frac{1}{2\sigma_{i,Y}}\log\frac{C}{\alpha} \rhp$ are non-overlapping.
\end{itemize}
Formally, we have the following theorem:
\begin{theorem}[Incremental learning under small spectral initialization]\label{thm_spec_inc}
    Suppose the $K$ non-zero singular values $\sigma_{1,Y},\sigma_{2,Y},\cdots,\sigma_{K,Y}$ of $Y$ are distinct and ordered in decreasing order. Let the spectral initialization $U(0)=\alpha^{1/2}\Phi\Sigma_{U_0}V_{U_0}^\top$ have a initialization shape $\Sigma_{U_0}^2=\diag\{\sigma_{i,0}\}_{i=1}^d$ with $\sigma_{i,0}>0,\forall 1\leq i\leq K$. Given any error tolerance $0<\varepsilon\leq \sigma_{K,Y}$, let $c_\varepsilon=\frac{\varepsilon}{\max_{i}\sigma_{i,0}}$, and $C_\varepsilon=\frac{\sigma_{1,Y}^2}{\varepsilon \min_{1\leq i\leq K}\sigma_{i,0}}$, Suppose $\alpha$ is sufficiently small such that $\alpha\leq c_\varepsilon$ and that $\frac{-\log\alpha+\log c_\varepsilon}{-\log\alpha+\log C_\varepsilon}> \max_{1\leq k\leq K-1}\frac{\sigma_{k+1}}{\sigma_k}$, then the time intervals
    \be
        \mathcal{I}_{k}:=\left[\frac{1}{2\sigma_{k,Y}}\log\frac{C_\varepsilon}{\alpha},\frac{1}{2\sigma_{k+1,Y}}\log\frac{c_\varepsilon}{\alpha}\right],\ 1\leq k \leq K\label{eq_time_int}
    \ee
    are all non-empty (we have let $\frac{1}{0}:=\infty$), and the $W(t)$ in Corollary \ref{col_spec_sol} satisfies that $\forall 1\leq k\leq K$,
    \begin{equation}
        \|W(t)-\hat{Y}_k\|\leq \varepsilon,\quad \forall t\in \mathcal{I}_{k}, \label{eq_inc_learning_spec}
    \end{equation}
    % there exists sufficiently small $\alpha$, and a corresponding sequence of time epochs $0<t_1<t_2<\cdots<t_{K-1}<t_K$ such that the $\W(t)$ in Proposition \ref{prop_sol} satisfies
    % \begin{equation}
    %     \|\W(t_k)-\Y_k\|\leq \varepsilon,\quad \forall 1\leq k\leq K-1;\qquad \|\W(t)-\Y\|\leq \varepsilon, \label{eq_inc_learning}
    % \end{equation}
    where $\hat{Y}_k:=\arg\min_{\mathrm{rank}(Z)=k}\|Y-Z\|_F$ is the best rank-$k$ approximation of $Y$.
\end{theorem}
\begin{proof}
    We have assumed small $\alpha$ such that $\frac{-\log\alpha+\log c_\varepsilon}{-\log\alpha+\log C_\varepsilon}> \max_{1\leq k\leq K-1}\frac{\sigma_{k+1}}{\sigma_k}$, which implies that for every $1\leq k\leq K$, $\mathcal{I}_{k}\neq \emptyset$. For a fixed $k$, choose any $t$ within this interval. Since $\Phi^\top W(t)\Phi$ and $\Phi^\top\hat{Y}_k\Phi$ are diagonal, write \eqref{eq_inc_learning_spec} as 
    \be
    |\sigma_{l,W}(t)-\sigma_{l,Y}|\!\leq\! \varepsilon,\forall l\leq k, \text{and } |\sigma_{l,W}(t)|\!\leq\! \varepsilon,\forall l>k.\label{eq_spec_goal}
    \ee
    Given the solution in Corollary \ref{col_spec_sol}, 
    we have, $\forall 1\leq l\leq k$,
    \begin{align}
        \sigma_{l,Y}\overset{(*)}{>} \sigma_{l,W}(t)&\overset{(*)}{\geq} \sigma_{l,W}\left(\frac{1}{2\sigma_{l,Y}}\log\frac{C_\varepsilon}{\alpha}\right)\nonumber\\
        &\overset{\eqref{eq_sv_traj_si_eval_b}}{=}\sigma_{l,Y}-\left(\sigma_{l,Y}-\frac{C_\varepsilon\sigma_{l,Y}\sigma_{l,0}}{\sigma_{l,Y}+(C_\varepsilon-\alpha)\sigma_{l,0}}\right)\nonumber\\
        & = \sigma_{l,Y}-\frac{\sigma_{l,Y}^2}{(\sigma_{l,Y}-\alpha\sigma_{l,0}) +C_\varepsilon\sigma_{l,0}}\nonumber\\
        &\overset{(\star)}{\geq} \sigma_{l,Y}-\frac{\sigma_{1,Y}^2}{C_\varepsilon\sigma_{l,0}}\geq \sigma_{l,Y}-\varepsilon\,,\label{eq_spec_bd1}
    \end{align}
    and $\forall k<l\leq K$,
    \begin{align}
        0\!\overset{(*)}{<}\!\sigma_{l,W}(t)&\!\overset{(*)}{\leq} \!\sigma_{l,W}\!\left(\frac{1}{2\sigma_{l,Y}}\log\frac{c_\varepsilon}{\alpha}\right)\nonumber\\
        &\overset{\eqref{eq_sv_traj_si_eval_s}}{=}\frac{c_\varepsilon\sigma_{l,Y}\sigma_{l,0}}{(\sigma_{l,Y}-\alpha\sigma_{l,0})+c_\varepsilon\sigma_{l,0}
        }\!\overset{(\star)}{\leq}\! c_\varepsilon\sigma_{l,0}\!\leq\! \varepsilon\,,\label{eq_spec_bd2}
    \end{align}
    where $(*)$ uses the monotonicity of $\sigma_{i,W}(t)$, and $(\star)$ uses the fact that $\sigma_{l,Y}-\alpha\sigma_{l,0}\geq 0,\forall l\leq K$, derived from $\alpha\leq c_\varepsilon$. Lastly, since we have assumed $\alpha\leq c_\varepsilon$, we have $\forall l>K$,
    \be
    0<\sigma_{l,W}(t)=\frac{\alpha \sigma_{l,0}}{1+2\alpha \sigma_{l,0}t}\leq \alpha \sigma_{l,0}\leq \varepsilon\,.\label{eq_spec_bd3}
    \ee 
    With \eqref{eq_spec_bd1}\eqref{eq_spec_bd2}\eqref{eq_spec_bd3}, the spectral bound in \eqref{eq_spec_goal} is verified.
\end{proof}

As shown in Theorem \ref{thm_spec_inc} and the preceding discussions, from a dynamical system perspective, the incremental learning phenomenon comes from some time-scale separation among dynamics corresponding to learning different components in the target matrix. By decreasing the initialization scale $\alpha$, these time-scale separations become more prominent, allowing one to find low-rank (best rank-$k$) approximations of the target matrix (See Figure \ref{fig:inc}) after one transitional phase concludes ($t=\frac{1}{2\sigma_{k,Y}}\log\frac{C_\varepsilon}{\alpha}$) and before another kicks in ($t=\frac{1}{2\sigma_{k+1,Y}}\log\frac{c_\varepsilon}{\alpha}$). Our Theorem exactly quantifies the scale needed to ensure the existence of such low-rank approximations (up to error $\varepsilon$), and the time intervals within which one can find these approximations.

\subsection{Incremental learning under general initialization}

When one moves to a general initialization $U(0)=\alpha^{1/2}\bar{U}_0$ for some $U_0$, the dynamics
\be
    \dot{\tilde{W}}=\tilde{W}\Sigma_Y+\Sigma_Y\tilde{W}-2\tilde{W}^2,\ \tilde{W}(0)=\alpha\tilde{W}_0\,,
\ee
can no longer be decomposed into scalar dynamics, where $\tilde{W}_0=\Phi^\top \bar{U}_0\bar{U}_0^\top \Phi$. Nonetheless, similar discussions in the previous section still work: Due to the small initialization scale, the dynamic evolution of the part in $\tilde{W}$ that learns top-$k$ singular components of $Y$ has a faster time scale than those that learn the rest of the singular values, and the interval $\mathcal{I}_k$ again separates the two dynamics. By carefully partitioning the closed-form solution obtained in \eqref{prop_cl_sol} and analyzing the growth of each block component, we have the following theorem suggesting the same incremental learning occur under general initialization:
\begin{theorem}[Incremental learning under general small initialization]\label{thm_gen_inc}
    Suppose the $K$ non-zero singular values $\sigma_{1,Y},\sigma_{2,Y},\cdots,\sigma_{K,Y}$ of $Y$ are distinct and ordered in decreasing order. Suppose the initialization of $U(0)=\alpha^{1/2}\bar{U}_0$ satisfies that $\tilde{W}_0=\Phi^\top \bar{U}_0\bar{U}_0^\top \Phi$ has an inverse denoted by $V$. Given some error tolerance $0<\varepsilon\leq \min\{\sigma_{K,Y},1\}$, let $c_\varepsilon=\frac{\varepsilon}{16M^2}$, $C_\varepsilon=\frac{16\sigma_{1,Y}^2M^2}{\varepsilon}$, where 
    % $M:=\max_{k1}\{\}$
    $M:=\max\{\|V\|,\|V^{-1}\|\}$.
    If the initialization scale $\alpha$ is sufficiently small so that  $\frac{-\log\alpha+\log c_\varepsilon}{-\log \alpha+\log C_\varepsilon}>\max_{1\leq k\leq K-1}\frac{\sigma_{k+1}}{\sigma_k}$ and $\alpha\leq \frac{c_\varepsilon}{M}$, then the time intervals $\mathcal{I}_k$ defined in \eqref{eq_time_int} are all non-empty, and the $W(t)$ in Proposition \ref{prop_cl_sol} satisfies that $\forall 1\leq k\leq K$,
    \begin{equation}
        \|W(t)-\hat{Y}_k\|\leq \varepsilon,\quad \forall t\in \mathcal{I}_{k}, \label{eq_inc_learning_gen}
    \end{equation}
    where $\hat{Y}_k:=\arg\min_{\mathrm{rank}(Z)=k}\|Y-Z\|_F$ is the best rank-$k$ approximation of $Y$.
\end{theorem}
\begin{proof}
    Similar to the proof of Theorem \ref{thm_spec_inc}, our assumption on $\alpha$ ensures every $\mathcal{I}_k$ is non-empty. To prove \eqref{eq_inc_learning_gen}, we show that for every choice of $k$, and for any $t_k\in\mathcal{I}_k$, we have $\|W(t_k)-\hat{Y}_k\|\leq \varepsilon$, the following proof proceeds given such a choice of $k$ and $t_k$.

    \noindent
    \textbf{Rewrite solution in blocks}: From \eqref{eq_riccati_sol} in Proposition \ref{prop_cl_sol},
    \begin{align}
        \tilde{W}(t_k)&\quad\ = \Side(t_k) \alpha\tilde{W}_0\left(I_n+\alpha\Mid (t_k) \tilde{W}_0\right)^{-1}\Side^\top(t_k)\nonumber\\
        &\overset{(\tilde{W}_0^{-1}=V)}{=}\alpha\Side(t_k) \left(V+\alpha\Mid (t_k)\right)^{-1}\Side^\top(t_k),\label{eq_pf_temp}
    \end{align}
    Let $\Sigma_k=\diag\{\sigma_{l,Y}\}_{l=1}^k$, and $\Sigma_k^c=\diag\{\sigma_{l,Y}\}_{l=k+1}^K$ (note that $\Sigma_k^c$ can be an empty matrix), write $V, \Side(t_k),\Mid(t_k)$ with the following block forms:
    \ben
        V\!=\!\bmt V_{11}&V_{12}\\ V_{12}^\top & V_{22}\emt\!, \Side(t_k)\!=\!\bmt \Side_1 &0\\ 0& \!\!\!\Side_2\emt\!, \Mid(t_k)\!=\!\bmt \Mid_1 &0\\ 0& \!\!\!\Mid_2\emt\!,
    \een
    where $\Side_1=e^{\Sigma_kt_k}$, $\Mid_1=\Sigma_k^{-1}(e^{2\Sigma_kt_k}-I_k)$,
    \ben
        \Side_2\!=\!\!\begin{bmatrix}
            e^{\Sigma_k^ct_k}\!\! & 0 \\
            0& \!\!I_{n\!-\!K}
        \end{bmatrix},\ \Mid_2\!=\!\!\begin{bmatrix}
            (\Sigma_k^c)^{-1}\!(e^{2\Sigma_k^ct_k}\!-\!I_{K\!-\!k})\!\!&0\\
            0& \!\!2I_{n\!-\!K}t_k
        \end{bmatrix}\!.
    \een
    Then we have 
    \begin{align}
        \tilde{W}(t_k)&\!=\! \alpha \bmt \Side_1^{-1}(V_{11} +\alpha \Mid_1)\Side_1^{-1} & \!\!\!\Side_1^{-1}V_{12}\Side_2^{-1}\\ \Side_2^{-1}V_{12}^\top \Side_1^{-1} &  \!\!\!\Side_2^{-1}(V_{22}+\alpha \Mid_2)\Side_2^{-1}\emt^{-1} \nonumber\\
        &\!=\! \bmt H_{11} & H_{12} \\ H_{12}^\top & H_{22}\emt \,,
    \end{align}
     where (by the inverse formula for a block matrix)
    \begin{align}
        H_{11}&\!=\!\alpha \Side_1\lp(V_{11} - V_{12}(V_{22}+\alpha \Mid_2)^{-1}V_{12}^\top)+\alpha \Mid_1\rp^{-1}\Side_1,\nonumber\\
        H_{12}&\!=\!- H_{11}\Side_1^{-1}V_{12}(V_{22}+\alpha \Mid_2)^{-1}\Side_2,\nonumber\\
        H_{22}&\!=\!\alpha \Side_2(V_{22}\!+\!\alpha \Mid_2)^{-1}\Side_2\!-\!\Side_2(V_{22}\!+\!\alpha \Mid_2)^{-1}V_{12}^\top\Side_1^{-1}H_{12}.\nonumber
    \end{align}
    Our goal is to show that 
    \be\|W(t_k)-\hat{Y}_k\|=\lV\Phi\bmt H_{11}-\Sigma_k & H_{12}\\ H_{12}^\top &H_{22}\emt\Phi^\top\rV\leq \varepsilon.\ee 
    Thus it suffices to show the following
    \be
        \|H_{11}-\Sigma_k\|\leq \frac{\varepsilon}{4},\ \|H_{12}\|\leq \frac{\varepsilon}{4},\ \|H_{22}\|\leq \frac{\varepsilon}{4}\,.\label{eq_gen_bd}
    \ee
    The rest of the proof is to show \eqref{eq_gen_bd} holds, for which we start by deriving norm bounds on $\Side_1$ and $\Side_2$, and then move to $H_{11}-\Sigma_k$, $H_{12}$ and $H_{22}$.

    \noindent
    \textbf{Norm bounds on $\Side_1$ and $\Side_2$}: From that $t_k\in\mathcal{I}_k$, we have
    \ben
        \Side_1^{-1}\!=\!e^{-\Sigma_kt_k}\!\preceq\! e^{-\Sigma_k\frac{1}{2\sigma_k}\log\frac{C_\varepsilon}{\alpha}}\preceq e^{-\frac{1}{2}\log\frac{C_\varepsilon}{\alpha}}I_k\!=\!\sqrt{\frac{\alpha}{C_\varepsilon}}I_k\,,
    \een
    and similarly since $e^{\Sigma_k^{c}t_k}\preceq e^{\Sigma_k^c\frac{1}{2\sigma_{k+1}}\log\frac{c_\varepsilon}{\alpha}}\preceq e^{\frac{1}{2}\log\frac{c_\varepsilon}{\alpha}}I_k$, we have $
        \Side_2\preceq \bmt \sqrt{\frac{c_\varepsilon}{\alpha}}I_{K-k} &0\\ 0& I_{n-K}\emt$. Together we have
    \be
        \|\Side^{-1}_1\|\leq \sqrt{\frac{\alpha}{C_\varepsilon}},\ \|\Side_2\|\leq \sqrt{\frac{c_\varepsilon}{\alpha}}\,.\label{eq_bd_s}
    \ee

    \noindent
    \textbf{Norm bounds on $H_{11}-\Sigma_k$, $H_{12}$ and $H_{22}$}: The primary focus is the bound on $\|H_{11}-\Sigma_k\|$, for the sake of simplicity, we let 
    $(V_{11} - V_{12}(V_{22}+\alpha \Mid_2)^{-1}V_{12}^\top):=\tilde{V}_1$, then
    \begin{align*}
        H_{11}&=\alpha\Side_1\lp \tilde{V}_1+\alpha \Mid_1\rp^{-1}\Side_1\\
        &=\lp \alpha^{-1}\Side_1^{-1}\tilde{V}_1\Side_1^{-1}+ \Side_1^{-1}\Mid_1\Side_1^{-1}\rp^{-1}\\
        &=\bigg( \underbrace{\alpha^{-1}\Side_1^{-1}\tilde{V}_1\Side_1^{-1}+ \Side_1^{-1}\Sigma_k^{-1}\Side_1^{-1}}_{:=\Delta}+\Sigma_k^{-1}\bigg)^{-1}\,.
    \end{align*}
    Notice that
    \ben
    0 \prec V_{11} - V_{12}V_{22}^{-1}V_{12}^\top\preceq V_{11} - V_{12}(V_{22}+\alpha \Mid_2)^{-1}V_{12}^\top\preceq V_{11}\,,
    \een
    from which we know $\|\tilde{V}_1\|\leq \|V_{11}\|\overset{(\textsf{a})}{\leq} M$. Then by \eqref{eq_bd_s},
    \ben
        \|\Delta\|\!\leq\! \frac{\|\Side_1^{-1}\!\|^2\|\tilde{V}_1\|}{\alpha}\!+\!\|\Side_1^{-1}\|^2\|\Sigma_k^{-1}\|\!\leq\! \frac{M\!+\!\sigma_{K,Y}^{-1}\alpha}{C_\varepsilon}\!\overset{(\textsf{b})}{\leq}\! \frac{2M}{C_\varepsilon}.
    \een
    As long as $\|\Sigma_k\|\|\Delta\|\overset{(\textsf{c})}{\leq}\frac{1}{2}$, we have, from~\cite{Horn:2012:MA:2422911},
    \begin{align*}
        \|H_{11}-\Sigma_k\|&\!=\!\|(\Delta+\Sigma_k^{-1})^{-1}-\Sigma_k\|\nonumber\\
        &\!\leq\! \frac{\|\Sigma_k\|^2\|\Delta\|}{1-\|\Sigma_k\|\|\Delta\|}\!\leq\! 2\sigma_{1,Y}^2\|\Delta\|\!\leq\! \frac{4\sigma_{1,Y}^2M}{C_\varepsilon}\!\leq\! \frac{\varepsilon}{4}\,,     
    \end{align*}
    the desired bound on $\|H_{11}-\Sigma_k\|$. Since we have assumed $\varepsilon\leq \sigma_{1,Y}$, this also suggests $\|H_{11}\|\leq \|\Sigma_k\|+\varepsilon\leq 2\sigma_{1,Y}$.

    For $H_{12},H_{22}$, we use the fact that $\|(V_{22}+\alpha \Mid_2)^{-1}\|\leq \|V_{22}^{-1}\|\overset{(\textsf{a})}{\leq}M$, and $\|V_{12}\|\overset{(\textsf{a})}{\leq} M$, together with \eqref{eq_bd_s}, to obtain
    \begin{align}
        \|H_{12}\|&\leq \|H_{11}\|\|\Side_1^{-1}\|\|V_{12}\|\|V_{22}^{-1}\|\|\Side_2\|\nonumber\\
        &\leq 2\sigma_{1,Y}M^2\sqrt{\frac{c_\varepsilon}{C_\varepsilon}}= \frac{\varepsilon}{8}\leq \frac{\varepsilon}{4}\,,
    \end{align}
    and
    \begin{align}
        \|H_{22}\|&\leq \alpha\|\Side_2\|^2\|V_{22}^{-1}\|+\|H_{12}\|\|\Side_1^{-1}\|\|V_{12}\|\|V_{22}^{-1}\|\|\Side_2\|\nonumber\\
        &\leq c_\varepsilon M +\frac{\varepsilon}{8}\sqrt{\frac{c_\varepsilon}{C_\varepsilon}}M^2\leq \frac{\varepsilon}{4}\,.
    \end{align}
    This finishes the proof. For the sake of readability, we omitted several arguments {\footnotesize$(\textsf{a})(\textsf{b})(\textsf{c})$} used in the last part of the proof, and we explain them here. ${\footnotesize(\textsf{a})}$: $\|V_{11}\|\leq M,\|V_{12}\|\leq M$ because the norm of submatrix is bounded by that of the full matrix, and $\|V_{22}^{-1}\|\leq M$ is due to the interlacing theorem~\cite{Horn:2012:MA:2422911} for principal submatrix. ${\footnotesize(\textsf{b})}$: we have $\sigma_{K,Y}^{-1}\alpha \leq \sigma_{K,Y}^{-1}\varepsilon \leq 1\leq M$. ${\footnotesize(\textsf{c})}$: with $\|\Delta\|\leq \frac{2M}{C_\varepsilon}$, it is easy to verify that $\|\Sigma_k\|\|\Delta\|\leq \sigma_{1,Y}\|\Delta\|\leq \frac{1}{2}$.
\end{proof}
The incremental learning results in Theorem \ref{thm_gen_inc} resemble the one in Theorem \ref{thm_spec_inc} for spectral initialization, except for some differences in the definitions of $c_\varepsilon, C_\varepsilon$. It characterizes the quantitative effect of the initialization scale in determining the time-scale separation among the learning dynamics for each target singular component, which provides guidance on how to utilize these results to design training algorithms that can learn low-rank approximations of the ground truth by early stopping. In addition, the proof of the symmetric factorization case might serve as a basis for broader settings.
\subsection{Remarks on other settings}
\subsubsection{Random initialization} A more natural setting is arguably the random initialization with small variance, which corresponds to having $U(0)=\alpha^{1/2} \bar{U}_0$ for some $\bar{U}_0$ whose entries are i.i.d. samples from a standard Gaussian (or other sub-Gaussian distributions). Applying Theorem \ref{thm_gen_inc} to this random initialization setting requires an extra concentration results on the matrix $\bar{W}_0=\bar{U}_0\bar{U}_0^\top$. If one can show that for $\delta\in(0,1)$, $\exists M_\delta>0$ such that $\prob\lp \max\{\|\bar{W}_0\|,\|\bar{W}_0^{-1}\|\}\leq M_\delta\rp\geq 1-\delta$, then one can define $c_\varepsilon,C_\varepsilon$ by $M_\delta$ and find initialization scale $\alpha$ accordingly, and Theorem \ref{thm_gen_inc} suggests that the incremental learning phenomenon as described by \eqref{eq_inc_learning_gen} happens with probability at least $1-\delta$. 
\subsubsection{Rank-deficient initialization} Our proof assumes an initialization shape $\bar{U}_0$ such that $\tilde{W}_0=\Phi^\top\bar{U}_0\bar{U}_0^\top \Phi$ is invertible. This is satisfied if $r<n$ when $\tilde{W}_0$ is always rank-deficient. This requires significant changes in the proof since now we should analyze $\tilde{W}(t_k) = \Side(t_k) \alpha\tilde{W}_0\big(I_n+\alpha\Mid (t_k) \tilde{W}_0\big)^{-1}\Side^\top(t_k)$ directly in \eqref{eq_pf_temp}. With Woodbury's matrix identity, one rewrite $\tilde{W}(t_k)$ as
\ben
    \alpha\Side(t_k)\tilde{U}\big( I\!-\!\tilde{U}^\top\! \Mid(t_k)\tilde{U} (I\!+\!\alpha\tilde{U}^\top \Mid(t_k)\tilde{U})^{-1}\big)\tilde{U}^\top\! \Side(t_k),
\een
where $\tilde{U}=\Phi^\top\bar{U}_0$. Detailed analysis on $\tilde{U}^\top \Mid(t_k)\tilde{U}$ leads to similar results as Theorem \ref{thm_gen_inc}, which is left to future work.

% \section{Discussions: Extension to General Problem Settings}\label{sec_discussions}

\section{Conclusion}\label{sec_conclusion}
In this paper, we develop a quantitative understanding of the incremental learning phenomenon in GF on symmetric matrix factorization problems through its closed-form solution. From a dynamical system perspective, the incremental learning phenomenon comes from some time-scale separation among dynamics corresponding to learning different components in the target matrix. By decreasing the initialization scale, these time-scale separations become more prominent, allowing one to find low-rank approximations of the target matrix.

Future work includes extending current results to asymmetric matrix factorization. Notably, for asymmetric matrix factorization,~\cite{pmlr-v139-tarmoun21a} shows that if the factors are initialized to satisfy some balancedness conditions, then the learning dynamics can be related to GF on a symmetric factorization problem. Moreover, under a small scale, any initialization is close to one that satisfies such balancedness conditions; therefore, the analysis for general initialization can be viewed as analyzing perturbed dynamics from a nomial dynamics with closed-form solution.
\bibliographystyle{ieeetr}
\bibliography{ref}

\end{document}